%% file: main.tex
\documentclass{article}
\usepackage{ijcai24}

\usepackage{times}
\usepackage{soul}
\usepackage{url}
\usepackage[hidelinks]{hyperref}
\usepackage[utf8]{inputenc}
\usepackage[small]{caption}
\usepackage{graphicx}
\usepackage{amsmath}
\usepackage{amsthm}
\usepackage{booktabs}
\usepackage{algorithm}
\usepackage{xspace}
\usepackage[switch]{lineno}
\usepackage[]{microtype}

\usepackage{yc_package}

\newcommand\modelfont[1]{\textsc{#1}}
\newcommand{\ourmethod}{\modelfont{DiveBatch}\xspace}

\newcommand{\oracle}{\modelfont{Oracle}\xspace}

\title{\ourmethod: Accelerating Model Training Through Gradient-Diversity Aware Batch Size Adaptation}

\author{
    ANONYMOUS AUTHOR(S)
}

\author{
Yuen Chen
\and
Yian Wang\and
Hari Sundaram\\
\affiliations
University of Illinois at Urbana-Champaign
\emails
\{yuenc2, yian3, hs1\}@illinois.edu
}

\begin{document}

\maketitle
\begin{abstract}
The goal of this paper is to accelerate the training of machine learning models, a critical challenge since the training of large-scale deep neural models can be computationally expensive. Stochastic gradient descent (SGD) and its variants are widely used to train deep neural networks. In contrast to traditional approaches that focus on tuning the learning rate, we propose a novel adaptive batch size SGD algorithm, \ourmethod, that dynamically adjusts the batch size. Adapting the batch size is challenging: using large batch sizes is more efficient due to parallel computation, but small-batch training often converges in fewer epochs and generalizes better. To address this challenge, we introduce a data-driven adaptation based on \textit{gradient diversity}, enabling \ourmethod to maintain the generalization performance of small-batch training while improving convergence speed and computational efficiency. Gradient diversity has a strong theoretical justification: it emerges from the convergence analysis of SGD. Evaluations of \ourmethod on synthetic and \textsc{CiFar}-10, \textsc{CiFar}-100, and \textsc{Tiny-ImageNet} demonstrate that \ourmethod converges significantly faster than standard SGD and AdaBatch (1.06 -- 5.0$\times$), with a slight trade-off in performance.
\end{abstract}
\input{contents/introduction}
\input{contents/backgrounds}

\input{contents/related_works}

\input{contents/methodology}

\input{contents/experiments}

\input{contents/limitation}

\input{contents/conclusion}
\clearpage
\bibliographystyle{named}
\bibliography{references,ref}
\appendix
\input{contents/appendix}

\end{document}

%% file: contents/introduction.tex
\section{Introduction}\label{sec:intro}

This paper aims to accelerate the training of machine learning models. The problem is important: training deep learning models can be computationally expensive with increasingly large datasets and model sizes. Significantly accelerating time to convergence can lead to substantial savings in computational resources and time, allowing practitioners to iterate more quickly and experiment with more models.

Training a model in a supervised learning setting requires solving an optimization problem that minimizes empirical risk. In large-scale machine learning, mini-batch stochastic gradient descent (SGD) and its variants are widely used methods for solving this problem due to the impracticality of computing gradients over the entire dataset~\citep{bottou_large-scale_2010, bottou_optimization_2018}. Accelerating the convergence speed of these optimization algorithms remains a key research focus. The choice of batch sizes significantly influences the performance and convergence speed of SGD. Larger batch sizes improve efficiency by enabling parallel gradient computations, resulting in faster iteration per epoch and better utilization of modern hardware capabilities. In contrast, small-batch training often converges in fewer epochs and achieves better generalization, as demonstrated in prior work~\citep{das_distributed_2016, keskar_improving_2017, masters_revisiting_2018}. These trade-offs present a significant challenge when selecting the appropriate batch size for optimizing SGD.

Previous studies have proposed heuristics for adaptive batch size methods to achieve the generalization performance of small-batch training without sacrificing the efficiency of large-batch training.\ \citet{smith_dont_2018} and \citet{devarakonda_adabatch_2018} demonstrate that starting with small-batch size and progressively increasing it during training can yield similar test performance to small-batch training. This approach allows models to benefit from the computational efficiency of large-batch training in the later stages while maintaining good generalization. However, these methods rely on heuristics and lack a theoretical foundation for determining batch size adjustments, leaving room for more principled approaches.

In this paper, we propose \ourmethod, an adaptive batch size scheme grounded in gradient diversity theory~\citep{yin_gradient_2018}. \ourmethod is a data-driven approach that dynamically adjusts the batch size based on the data characteristics. At the end of each epoch, we update the batch size to be proportional to the gradient diversity of the model parameters.

We investigate the role of gradient diversity in SGD convergence through experiments on synthetic and real-world datasets for both convex and nonconvex optimizations. 
These experiments demonstrate that starting with small-batch size and then using gradient diversity to adjust the batch size during training can significantly accelerate convergence. We show that \ourmethod outperforms fixed large-batch SGD while matching the performance of small-batch SGD.
Further, we compare \ourmethod against adaptive batch size methods proposed by \citet{devarakonda_adabatch_2018} through training ResNet-20~\citep{he2016deep}. Our results show that \ourmethod converges to the final result faster than Adabatch and small-batch SGD. To summarize, our main contributions are:

\begin{itemize}
    \item \textbf{Data-adaptive batch size Tuning:} To the best of our knowledge, \ourmethod is the first algorithm to leverage gradient diversity for tuning batch size.\ Unlike prior state-of-the-art methods that rely on heuristic schedules for batch size adjustment. \ourmethod introduces an innovative and adaptive approach where the batch size is adjusted in proportion to the gradient diversity. 
    \item \textbf{Accelerated Convergence:} \ourmethod achieves accelerated convergence compared to fixed-batchsize SGD and state-of-the-art adaptive batch size method~\citep{devarakonda_adabatch_2018} while maintaining comparable accuracy. It achieves within 1\% of the final performance at a speed 1.06 to 5$\times$ faster than small-batch SGD and AdaBatch, allowing researchers to iterate rapidly over different potential models in the early training stage.
\end{itemize}

%% file: contents/backgrounds.tex
\section{Background}

Consider a dataset consisting of $n$ training samples  $\mathcal{S} = \{\mathbf{z}_1, \dots, \mathbf{z}_n\}$ i.i.d. sampled from an unknown distribution $\mathcal{D}$ over the sample space $\mathcal{Z}$. The goal is to learn a model parameter $\boldsymbol{\theta} \in \mathbb{R}^d$ to minimize the \textit{expected risk} of the model on the distribution $\mathcal{D}$, $\mathcal{L}\left(\boldsymbol{\theta} \right) := \mathbb{E}_{\mathbf{z} \sim \mathcal{D}}\left[ \ell \left(\boldsymbol{\theta}; \mathbf{z}\right)\right]$, where $\ell$ is the sample loss function. As $\mathcal{D}$ is unknown, we instead minimize the empirical risk on the training samples,
\begin{equation}\label{eq:erm}
    \min_{\boldsymbol{\theta} \in \mathbb{R}^d} \left\{ \mathcal{L}_\mathcal{S} \left( \boldsymbol{\theta} \right) = \frac{1}{n}\sum_{i = 1}^n \ell\left(\boldsymbol{\theta}; \mathbf{z}_i\right) \right\}.
\end{equation}

In minibatch SGD, we estimate the gradient of the entire dataset $\nabla \mathcal{L}_{\mathcal{S}}(\boldsymbol{\theta})$ using the gradient of a minibatch of data points sampled from $\mathcal{S}$, $\frac{1}{|\mathcal{B}|} \sum _{i \in \mathcal{B}}\nabla_{\boldsymbol{\theta}} \ell(\boldsymbol{\theta}; \mathbf{z}_i)$, where $\mathcal{B}$ denote a minibatch of indices and $|\mathcal{B}| << n$.

\subsection{Effect of Mini-batch Size}
Assume $\boldsymbol{z}_i \overset{\text{i.i.d.}}{\sim} \mathcal{D}$ for $i \in [n]$, a mini-batch $\mathcal{B}$ of size $m$ sampled uniformly at random from $[n]$, and that the sample gradient with respect to $\boldsymbol{\theta}$ has finite covariance, i.e., $\Sigma\left( \nabla_{\boldsymbol{\theta}} \ell(\boldsymbol{\theta}; \boldsymbol{z}_i)\right)$ (abbreviated as $\Sigma$).

The covariance of $g(\boldsymbol{\theta})$ is given by $\frac{\Sigma}{m}$, which decreases linearly as $m$ increases. Conversely, the number of gradient computations increases linearly with $m$. Thus, the mini-batch size $m$ plays a critical role in balancing the trade-off between convergence speed and computation costs. When $m = 1$, we recover SGD, which has the highest variance but the lowest computational cost per iteration. Conversely, when $m = n$, we recover gradient descent, which has the smallest variance but the highest computation cost per iteration.

Although most theoretical analyses assume that mini-batches are sampled i.i.d. with replacement, practical training of a model using mini-batch SGD differs. In practice, the dataset $\mathcal{S}$ is first partitioned into $\lceil \frac{n}{m} \rceil$ minibatches. One complete pass through these partitions, referred to as one \textit{epoch}, ensures that all data points are seen once. 

In a distributed setting, the $m$ samples within a mini-batch can be computed in parallel. Therefore, using a small-batch size $m$ leads to a large number of mini-batches, as $\lceil \frac{n}{m} \rceil$ is large, and these mini-batches must be processed sequentially, limiting the benefit from parallelization. On the other hand, large-batch sizes are more efficient, as each epoch requires fewer sequential computations. However, large-batch SGD is less robust as it often finds sharp minima~\citep{keskar_large-batch_2017}, leading to performance degradation on the test dataset.

\subsection{Gradient Diversity}
Ideally, one would hope that mini-batch SGD achieves linear speed up compared to SGD, as the computation costs also increase linearly. The theoretical analysis in \citet{yin_gradient_2018} shows that such linear convergence is dominated by a quantity termed \textit{Gradient Diversity}.

\begin{definition}[gradient diversity~\citep{yin_gradient_2018}]\label{def:gradient_diversity}
The gradient diversity of a dataset $\mathcal{S}$ with respect to $\boldsymbol{\theta}$ is :

\begin{equation*}
    \Delta_\mathcal{S} \left(\boldsymbol{\theta} \right) := \frac{\sum_{i = 1}^n \|\nabla_{\boldsymbol{\theta}} \ell \left(\boldsymbol{\theta}; \boldsymbol{z}_i \right)\|^2_2}{\left\|\sum_{i = 1}^n \nabla_{\boldsymbol{\theta}} \ell \left(\boldsymbol{\theta}; \boldsymbol{z}_i \right)\right\|^2_2}
\end{equation*}
\end{definition}

In particular, \citet{yin_gradient_2018} shows if the batch size $m$ is proportional to $n \Delta_{\mathcal{S}} \left(\boldsymbol{\theta} \right)$, given fixed $\boldsymbol{\theta}^{t}$, one mini-batch SGD update reduces distance between $\boldsymbol{\theta}^{t+1}$ and the minimizer $\boldsymbol{\theta}^*$ by approximately $m$ times more than one SGD update. We leave the details of this result to \cref{app:yin_lemma1} in the appendix.

In mini-batch SGD, only the gradients of a subset of $\mathcal{S}$ are computed, making $\Delta_{\mathcal{S}}\left(\boldsymbol{\theta} \right)$ generally inaccessible. Another challenge arises because, when computing $\boldsymbol{\theta}^{t+1}$, we evaluate the gradients with respect to $\boldsymbol{\theta}^t$. Tuning the batch size based on $\Delta_{\mathcal{S}}\left(\boldsymbol{\theta} \right)$ would require knowing the gradients at $\boldsymbol{\theta}^t$ before the gradient evaluation begins, leading to a circular dependency. Therefore, to the best of our knowledge, no previous works have used gradient diversity for batch size tuning.

%% file: contents/related_works.tex
\section{Related Work}
Stochastic gradient descent (SGD) has been an active research area since its inception by \citet{robbins_stochastic_1951}, owing to its its computation efficiency and applicability to large-scale problems. Various efforts have focused on improving SGD-based algorithms to speed up convergence and optimize resource utilization.\ A key focus of SGD research is the selection of hyperparameters, particularly the learning rates and batch sizes.

\paragraph{Adaptive Learning Rates}
Classic convergence analysis of SGD relies on a sequence of learning rates satisfying the Robbins-Monro conditions,

\begin{equation}
    \sum_{t=1}^{\infty} \eta_t = \infty \quad
    \text{and} \quad \sum_{t=1}^{\infty} \eta^2_t < \infty.
\end{equation} Recent advances in SGD have predominantly focused on adjusting the learning rate throughout the training process~\citep{duchi_adaptive_2011, horvath_adaptive_2022, bottou_optimization_2018, hinton_neural_2012, kingma_adam_2017, li_convergence_2019, zeiler_adadelta_2012} while fixing the batch size. Typically, the learning rate decreases over time, so that at the early stage of training, the model can first find a neighborhood of local minimum with fewer iterations. At the later stage of training, the model can converge more precisely by taking smaller steps, thereby avoiding overshooting the minimum. While some methods, such as \citet{horvath_adaptive_2022}, use gradient diversity to tune the learning rate, these approaches are orthogonal to ours. \ourmethod focuses on adapting the batch size dynamically, rather than adjusting the learning rate. As such, comparisons with adaptive learning rate methods fall outside the scope of this work. Nevertheless, \ourmethod could be combined with such techniques in future research to further accelerate the training process.

However, as \citet{johnson_accelerating_2013} noted, decreasing the learning rate can slow the convergence of SGD algorithms. Conversely, large-batch training is often desirable in practice because it reduces the variance of gradient estimates and allows for greater parallelization across multiple machines, leading to fewer parameter updates needed to train a model.

Despite its practical advantages, training with large, fixed batch sizes often results in finding sharp minima, which can which can degrade generalization performance on test data~\citep{masters_revisiting_2018, keskar_large-batch_2017, lau_adadagrad_2024}. Fortunately, there exists a duality between learning rate and batch sizes; \citet{goyal_accurate_2017} suggests that when the batch size is increased by a factor of $\alpha$, the learning rate should be scaled by $\alpha$ to maintain the same \textit{effective learning rate}. \citet{smith_dont_2018} demonstrated that one can match learning curves of learning rate decay by gradually increasing batch sizes, showing the possibility of improving training efficiency by increasing batch sizes without sacrificing generalization performances.

\paragraph{Adaptive Batch Sizes}
Building on this insight, adaptive batch size methods aim to address the issues associated with large-batch training and the slowing down of convergence due to decreased learning rates. Instead of decaying the learning by $\alpha$  (for some $0 < \alpha < 1$), \citet{smith_dont_2018} propose scaling the batches by $\frac{1}{\alpha}$, while \citet{devarakonda_adabatch_2018} both decreases learning rates and increase the batch sizes. Beyond mimicking the training trajectory of the decaying learning rate, another approach is to adjust batch sizes according to, for example, the variance of the stochastic gradients~\citep{balles_coupling_2017}, gradient noises~\citep{mccandlish_empirical_2018}, or gradient cosine similarities~\citep{qin_simigrad_2021}. We propose using gradient diversity to dynamically adjust the batch sizes. The advantage of this approach is that gradient diversity is a quantity that emerges from the convergence analysis of SGD and is therefore theoretically grounded. Additionally, techniques like quantization~\citep{alistarh_qsgd_2017} and stochastic gradient Langevin dynamics~\citep{welling_bayesian_2011} can be seamlessly integrated with our method to increase gradient diversity~\citep{yin_gradient_2018}, facilitating the use of larger batch sizes and improving overall training efficiency.

%% file: contents/methodology.tex
\section{Methodology}
Motivated by the convergence analysis building upon gradient diversity, we propose \ourmethod, a mini-batch SGD algorithm where the batch sizes are updated based on the estimated gradient diversity over mini-batches of the current epoch.

\paragraph{Notations} We denote one pass through the training dataset as an \emph{epoch}. Given a mini-batch size $m_k$, the $k^{th}$ epoch involves $\lceil \frac{n}{m_k} \rceil$ parameter updates (iterations). Let $t_k := \sum_{j = 1}^{k-1} \lceil \frac{n}{m_j} \rceil$ be the number of parameter updates that have already occurred, and $\boldsymbol{\theta}^{t_k}$ the parameter at the start of epoch $k$.

For the $j^{th}$ parameter update within epoch $k$, we store the sum of individual gradients' $L_2$-norm, $\sum_{i \in \mathcal{B}_j} \|\nabla_{\boldsymbol{\theta}} \ell \left(\boldsymbol{\theta}^{t_k + j-1}; \boldsymbol{z}_i\right)\|^2_2$, and the sum of the gradients $\sum_{i \in \mathcal{B}_j} \nabla_{\boldsymbol{\theta}} \ell \left(\boldsymbol{\theta}^{t_k + j-1}; \boldsymbol{z}_i\right)$. At the end of epoch $k$, we accumulate the sum of gradients and the sum of gradient norms to estimate the gradient diversity of the parameter after finishing epoch $k$, i.e., $\Delta_\mathcal{S} \left(\boldsymbol{\theta}^{t_{k + 1}}\right) $, 

\begin{definition}[Estimated Gradient Diversity]
Let subset $\mathcal{B}_1, \dots, \mathcal{B}_{\lceil \frac{n}{m_k} \rceil}$ be the mini-batches of epoch $k$.  We define the estimated the gradient diversity w.r.t. $\boldsymbol{\theta}^{t_k + \lceil \frac{n}{m_k} \rceil}$, i.e., $\boldsymbol{\theta}^{t_{k + 1}}$:
\begin{equation*}
    \hat{\Delta}_{\mathcal{S}}\left(\boldsymbol{\theta}^{t_{k + 1}}\right) := \frac{\sum_{j=1}^{{\lceil \frac{n}{m_k} \rceil}} \sum_{i \in \mathcal{B}_j} \|\nabla_{\boldsymbol{\theta}} \ell \left(\boldsymbol{\theta}^{t_k + j-1}; \boldsymbol{z}_i\right)\|^2_2}{\|\sum_{j=1}^{{\lceil \frac{n}{m_k} \rceil}} \sum_{i \in \mathcal{B}_j} \nabla_{\boldsymbol{\theta}} \ell \left(\boldsymbol{\theta}^{t_k + j-1}; \boldsymbol{z}_i\right)\|^2_2}.
\end{equation*}
\end{definition}
We update the batch size to be proportional to $\hat{\Delta}_{\mathcal{S}}\left(\boldsymbol{\theta}^{t_{k + 1}}\right)$. The full algorithm is in Algorithm~\ref{alg:ouralgo}.

\begin{algorithm}
    \caption{\ourmethod}\label{alg:ouralgo}
    \begin{algorithmic}[1]
    \State \textbf{Input:} Dataset $\mathcal{S}$ of size $n$, initial point $\boldsymbol{\theta}_0$, learning rate $\eta$, initial batch size $m_0$,  number of epochs $K$, maximum batch size $m^{max}$, $\delta$ from \ref{app:yin_lemma1}.
    \State \textbf{Output: Optimized $\boldsymbol{\theta}$}
    \For{$k = 0, \dots, K-1$}
        \For {$j = 1, \dots, \lceil \frac{n}{m_k} \rceil$}
        \State Sample mini-batch $\mathcal{B}_j$ of size $m$ from $\mathcal{D}$
        \State $\nabla_{\boldsymbol{\theta}} \mathcal{L}_{\mathcal{B}_j}(\boldsymbol{\theta}^{t_k + j - 1}) \leftarrow \sum_{i \in \mathcal{B}_j} \nabla_{\boldsymbol{\theta}} \ell(\boldsymbol{\theta}^{t_k + j - 1}; \boldsymbol{z}_i)$
        \State $G \leftarrow \sum_{i \in \mathcal{B}_j} \|\nabla_{\boldsymbol{\theta}} \ell(\boldsymbol{\theta}^{t_k + j - 1}); \boldsymbol{z}_i)\|^2_2$
        \State $\boldsymbol{\theta}^{t_k + j} = \boldsymbol{\theta}^{t_k + j - 1} - \frac{\eta}{m_k}\nabla_{\boldsymbol{\theta}} \mathcal{L}_{\mathcal{B}_j}(\boldsymbol{\theta}^{t_k + j - 1})$
            \EndFor
        \State $\hat{\Delta}_{\mathcal{S}}(\boldsymbol{\theta}^{t_{k + 1}}) \leftarrow \frac{\sum_{j=1}^{\lceil \frac{n}{m_k} \rceil} \sum_{i \in \mathcal{B}_j} \|\nabla_{\boldsymbol{\theta}} \ell(\boldsymbol{\theta}^{t_k + j - 1}; \boldsymbol{z}_i)\|^2_2}{\|\sum_{j=1}^{\lceil \frac{n}{m_k} \rceil} \sum_{i \in \mathcal{B}_j} \nabla_{\boldsymbol{\theta}} \ell(\boldsymbol{\theta}^{t_k + j - 1}; \boldsymbol{z}_i)\|^2_2}$
            \State $m_{k+1} \leftarrow \min \left\{m^{max}, \delta \cdot {n} \cdot \hat{\Delta}_{\mathcal{S}}\left(\boldsymbol{\theta}^{t_{k + 1}}\right) \right\}$

    \EndFor
    \end{algorithmic}
\end{algorithm}

%% file: contents/experiments.tex
\section{Experiments}
We validate \ourmethod through experiments on both synthetic and three popular real-world datasets: 
\textsc{CiFar}-10, \textsc{CiFar}-100~\citep{krizhevsky2009learning}, and \textsc{Tiny-ImageNet}~\citep{le_tiny_nodate}. With synthetic experiments, we aim to understand interactions among gradient diversity approximation, batch sizes, and generalization performance. In experiments on real-world datasets, we compare the performance of \ourmethod against minibatch SGD with fixed batch sizes and AdaBatch~\citep{devarakonda_adabatch_2018}. 

We implement \ourmethod using PyTorch version 2.5.1 with GPU acceleration.
Our experimental platform consists of 4 NVIDIA Tesla A100 GPUs.

\begin{figure*}[ht!]
    \centering
    \begin{subfigure}[b]{0.3\textwidth}
        \centering
        \includegraphics[width=\textwidth]{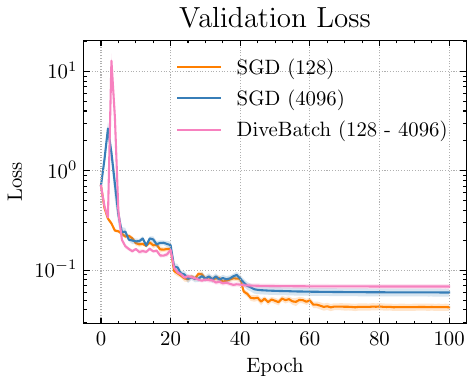}
        \label{fig:val_losses}
    \end{subfigure}\hspace{15pt}
    \begin{subfigure}[b]{0.3\textwidth}
        \centering
        \includegraphics[width=\textwidth]{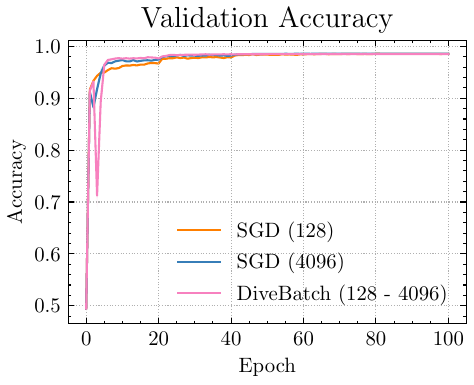}
        \label{fig:val_accs}
    \end{subfigure}

    \begin{subfigure}[b]{0.3\textwidth}
        \centering
        \includegraphics[width=\textwidth]{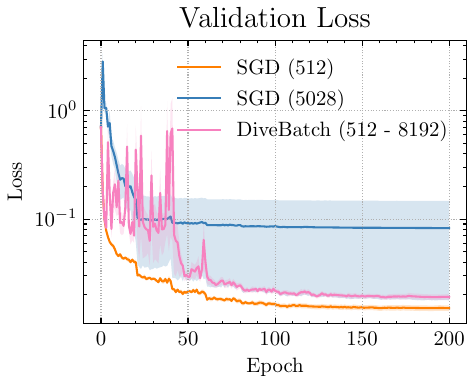}
        \label{fig:nonconvex_val_losses}
    \end{subfigure}\hspace{15pt}
        \begin{subfigure}[b]{0.3\textwidth}
        \centering
        \includegraphics[width=\textwidth]{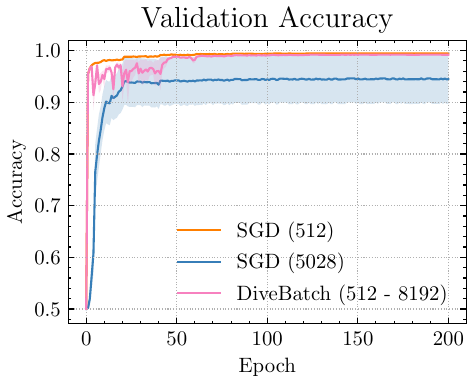}
        \label{fig:nonconvex_val_accs}
    \end{subfigure}
    \caption{Validation losses and accuracies of logistic regression models under convex and nonconvex setting. \textbf{Convex case} \textit{(Top)}: In the parentheses are the batch sizes for SGD and the initial batch size - end batch size for \ourmethod. \ourmethod shows a faster stabilization in validation loss compared to SGD (128) and SGD (4096). Validation accuracy for \ourmethod is similar to SGD. \textbf{Nonconvex case} \textit{(Bottom)}: With \ourmethod, we can increase the batch size to 5028 without much performance degradation, whereas for SGD, using a fixed batch size of 4096 results in significantly worse validation loss and accuracies compared to training with a batch size of 512.}

    \label{fig:combined_results}
\end{figure*}

\subsection{Synthetic Experiments}
\subsubsection{Setup} The synthetic dataset is generated using a linear combination of features with added random noise. For a single data point, we first create a feature vector $x$ by drawing each feature component uniformly from \([-1, 1]^d\). This feature vector is then multiplied by a randomly generated weight vector \(\mathbf{w}^* \in \mathbb{R}^d\), drawn from a standard normal distribution, to create a linear combination. Random noise, scaled by a specified noise factor, is added to this linear combination to simulate real-world data imperfections. The resulting value is passed through a sigmoid function to obtain a probability. The label is then generated by thresholding this probability at 0.5: if the probability is greater than 0.5, the label is set to 1; otherwise, it is set to 0. Mathematically, we express the process as:
\begin{equation}\label{eq:data_gen}
y = \mathds{1} \left\{\sigma\left( \mathbf{w}^* \cdot \mathbf{x}  + \epsilon \right) > 0.5 \right\} \text{, } \epsilon \sim \mathcal{N}\left(\boldsymbol{0}, 0.1 \cdot \boldsymbol{I}\right)
\end{equation} 

We set $d = 512$ and generate 20,000 data points, dividing them into training (80\%) and validation (20\%) sets.

\paragraph{Models} We run two model architectures to cover both convex and nonconvex optimization. For the convex case, we train logistic regression models with dimension $d$, consistent with the data generation, to learn $\boldsymbol{\theta}^* := \left\{\mathbf{w}^*, b^* \right\}$ where $b^* = 0$. As for the non-convex case, we train 2-layer MLPs with the same number of parameters as the logistic regression models. As the synthetic experiments aim to understand the behavior of \ourmethod, we use SGD as our baseline.

\paragraph{Hyperparameters} We begin by performing a grid search over the learning rates $\{0.001, 0.01, 0.1, 1, 2, 4, 8, 16\}$ and batch sizes $\{32, 64, 128, 256, 512\}$ to find the optimal settings for SGD that yield the highest validation accuracy after 100 epochs. We denote these optimal values as $\eta^\text{sgd}$ and $m^\text{sgd}$, respectively. We search over small-batch sizes ($\leq 512$) as the goal is to match the performance of adaptive batch sizes with small-batch training. 

For \ourmethod, we use $\eta^\text{sgd}$ as the initial learning rate. Additionally, we search over $\delta$ in $\{0.001, 0.01, 0.1, 1\}$ and $m^{\max}$ in $\{2048, 4096, 8192\}$. For comparison, we run SGD with batch size $m^{\max}$. When increasing the batch size to $m_{k+1}$, we scale the learning rate by $\frac{m_{k+1}}{m_k}$ to maintain the learning rate and batch size ratio $\frac{\eta_{k+1}}{m_{k+1}} = \frac{\eta^\text{sgd}}{m^\text{sgd}}$ as suggested by \citet{goyal_accurate_2017}.

In addition to scaling the learning rate by batch sizes, we apply a learning rate schedule to each method that scales the learning rate by $0.75$ every 20 epochs, consistent with experiments in \citet{devarakonda_adabatch_2018}. We run each method for 10 trials and report the averaged results.

\subsubsection{Results and Discussions}

\begin{figure*}[t!]
    \centering
    \begin{subfigure}[b]{0.29\textwidth}
        \centering
        \includegraphics[width=\textwidth]{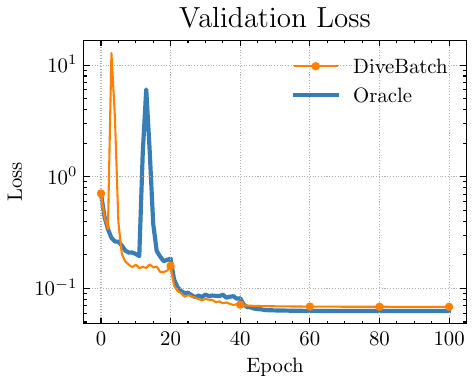}
        \label{fig:oracle_convex_val}
    \end{subfigure}
    \hspace{10pt}
    \begin{subfigure}[b]{0.3\textwidth}
        \centering
        \includegraphics[width=\textwidth]{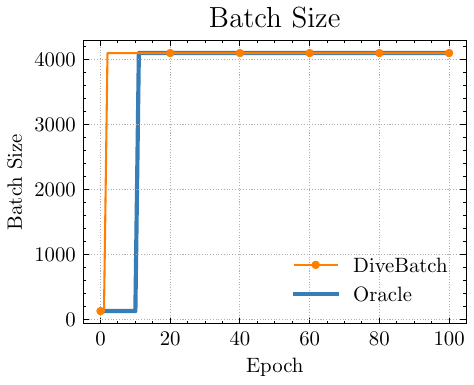}
        \label{fig:convex_train_losses}
    \end{subfigure}
    \hspace{10pt}
    \begin{subfigure}[b]{0.29\textwidth}
        \centering
        \includegraphics[width=\textwidth]{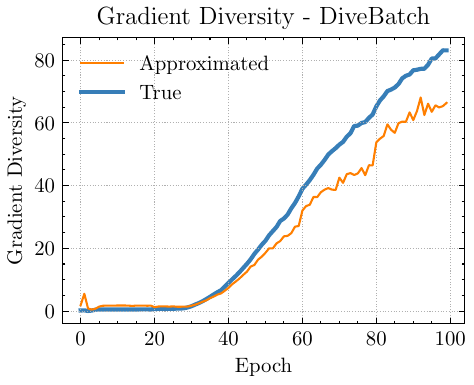}
        \label{fig:convex_grad_div}
    \end{subfigure}
    
    \begin{subfigure}[b]{0.29\textwidth}
        \centering
        \includegraphics[width=\textwidth]{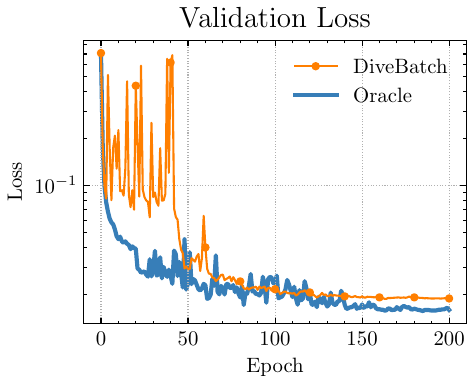}
        \label{fig:oracle_nonconvex_val}
    \end{subfigure}\hspace{12pt}
    \begin{subfigure}[b]{0.3\textwidth}
        \centering
        \includegraphics[width=\textwidth]{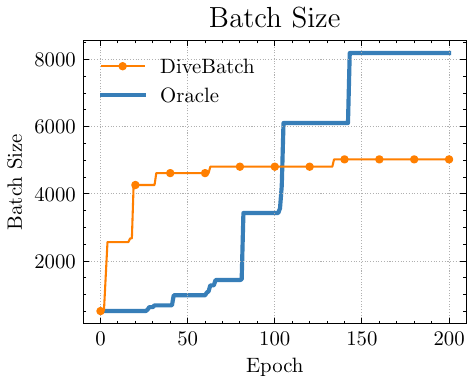}
        \label{fig:nonconvex_train_accs}
    \end{subfigure}\hspace{10pt}
    \begin{subfigure}[b]{0.29\textwidth}
        \centering
        \includegraphics[width=\textwidth]{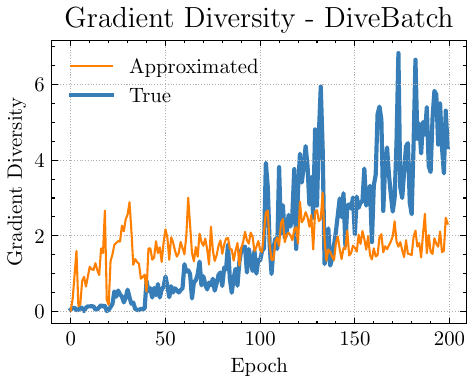}
        \label{fig:nonconvex_grad_div}
    \end{subfigure}
    \caption{Validation losses and batch size progression of \oracle and \ourmethod. \textbf{Convex case} \textit{(Top)}: The batch sizes of \oracle and \ourmethod closely match. \textbf{Nonconvex case} \textit{(Bottom)}: The batch size of \oracle increases to $m^{\max}$ gradually while \ourmethod uses $m^{\max}$ after only a few epochs, and \oracle achieves lower validation accuracies.}
    \label{fig:oracle_combined_results}
\end{figure*}

\paragraph{Convex Case} For the choice of $\delta$ in the batch size update, we performed a grid search over $\{0.001, 0.01, 0.1, 1\}$. Surprisingly, a large $\delta$ value ($\delta > 0.1$) performs better than smaller values. As long as $ n >> m^{\max}$, according to line 11 in \cref{alg:ouralgo}, using a large $\delta$ means the algorithm started training with $m^{\max}$ after few epochs. In other words, we find that starting with a small-batch size for a few epochs, then training on large-batches improves convergence. In \cref{fig:combined_results} we report the results of SGD with $m^\text{sgd} = 128$ and $m^{\max} = 4096$ and observe that the validation loss of \ourmethod drops more rapidly after approximately 10 epochs and stabilizes by epoch 40, outperforming the SGD baselines in speed.  While \ourmethod converges to a higher validation loss, it achieves its final validation accuracy after around 20 epochs, matching the ending accuracy of the SGD baselines. Note that the spike in validation loss (dip in validation accuracy) of \ourmethod results from the increase in learning rate and batch size.

\paragraph{Non-convex Case} We follow the same hyperparameter search procedure as in the convex case. With the optimal choice of $\delta = 0.1$, we observe that the batch size does not always increase to $m^{\max}$ which is 8192. Instead, the average maximum batch size selected by \ourmethod is $5028$. The bottom two figures in \cref{fig:combined_results}
 show that training with a fixed large-batch size (5028) leads to much higher average loss and lower average accuracy. This occurs because, in some of the 10 trials, the losses do not converge. In contrast, our adaptive batch size scheme allows training with a large-batch size of 5028 while achieving a slightly higher validation loss but maintaining comparable validation accuracy to small-batch SGD (512).

\paragraph{The Effect of Gradient Diversity Estimate Quality} To investigate the effect of using gradient estimation $\hat{\Delta}_{\mathcal{S}}(\boldsymbol{\theta}^{t_{k + 1}})$ in place of the true gradient diversity $\Delta_\mathcal{S} \left(\boldsymbol{\theta} \right)$, we design an \oracle algorithm that is almost identical to \cref{alg:ouralgo} but with a key difference: instead of using the approximation $\hat{\Delta}_{\mathcal{S}}\left(\boldsymbol{\theta}^{t_{k + 1}}\right)$, \oracle computes the exact gradient diversity ${\Delta}_{\mathcal{S}}\left(\boldsymbol{\theta}^{t_{k + 1}}\right)$ of the entire dataset after every epoch. \oracle uses the same hyperparameters as \ourmethod.

The top three figures of \cref{fig:oracle_combined_results} show that in the convex case, $\hat{\Delta}_{\mathcal{S}}\left(\boldsymbol{\theta}^{t_{k + 1}}\right)$ is a good approximation for ${\Delta}_{\mathcal{S}}\left(\boldsymbol{\theta}^{t_{k + 1}}\right)$. While the approximation error increases after epoch 40, this does not impact the batch sizes because they have already reached $m^{\max}$. The batch size schedules of \ourmethod closely match those of \oracle, with both methods agreeing to increase the batch size to $m^{\max}$. 
In the nonconvex case, however,  \oracle gradually increases the batch sizes to $m^{\max} = 8192$, while \ourmethod rapidly scales the batch size during the early stages of training before stabilizing at 5028. During the initial training stage (before epoch 50), when \ourmethod uses larger batch sizes, \oracle's validation losses drop faster and are more stable. In the convex case, the similarity in batch size schedules results in comparable final validation losses. In contrast, in the nonconvex case, \oracle’s dynamic adjustment of batch sizes based on true gradient diversity enables it to converge to a lower validation loss, demonstrating the advantage of using exact gradient diversity.

The gradient diversity plots reveal that both the approximated and true gradient diversities increase steadily in the convex case and have larger magnitudes than in the nonconvex case. This difference highlights the complexity of optimizing nonconvex objectives compared to convex ones.

\begin{figure*}[t!]
    \centering
    \begin{subfigure}[b]{0.33\textwidth}
        \centering
        \includegraphics[width=\textwidth]{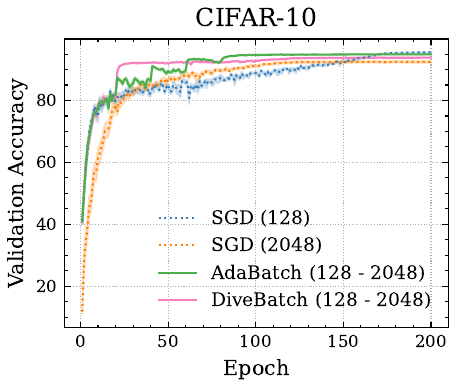}
    \end{subfigure}\hspace{0pt}
    \begin{subfigure}[b]{0.33\textwidth}
        \centering
        \includegraphics[width=\textwidth]{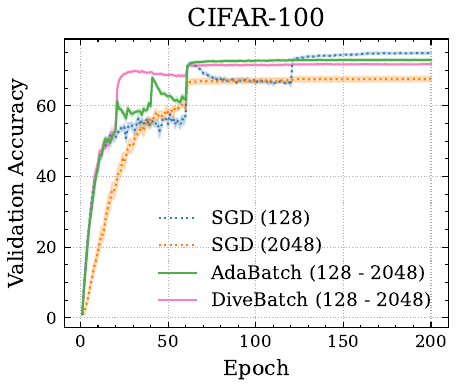}
    \end{subfigure}\hspace{0pt}
    \begin{subfigure}[b]{0.33\textwidth}
    \centering
    \includegraphics[width=\textwidth]{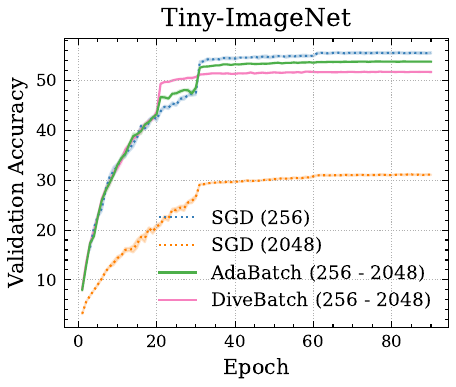}
    \end{subfigure}\hspace{15pt}
    \caption{Validation accuracy on \textsc{CiFar}-10, \textsc{CiFar}-100 and \textsc{Tiny-ImageNet}. In the parentheses are the batch sizes for SGD and initial batch size - end batch size for Adabatch and \ourmethod. The accuracies of \ourmethod on  \textsc{CiFar}-10 and textsc{Tiny-ImageNet} stabilize at around epoch 30, faster than AdaBatch, but result in lower accuracies. The plots show the average validation accuracy over 5 trials.}
    \label{fig:real_combined_acc}
    \centering
    \begin{subfigure}[b]{0.33\textwidth}
        \centering
        \includegraphics[width=\textwidth]{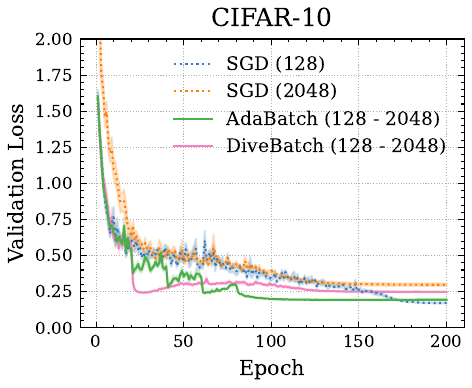}
    \end{subfigure}\hspace{6pt}
    \begin{subfigure}[b]{0.31\textwidth}
        \centering
        \includegraphics[width=\textwidth]{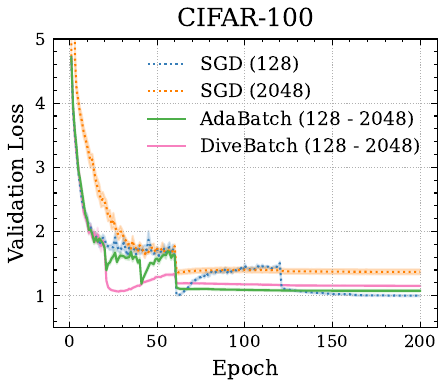}
    \end{subfigure}\hspace{0pt}
    \begin{subfigure}[b]{0.32\textwidth}
    \centering
    \includegraphics[width=\textwidth]{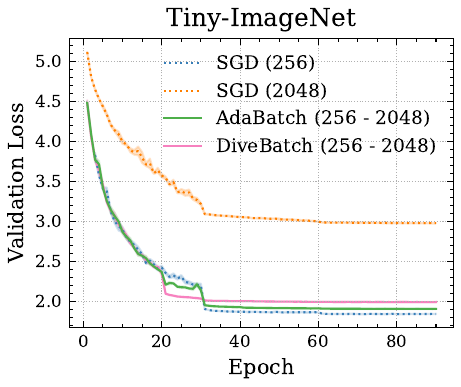}
    \end{subfigure}
    \caption{Validation loss on \textsc{CiFar}-10, \textsc{CiFar}-100 and \textsc{Tiny-ImageNet}. The loss of \ourmethod stabilizes faster than AdaBatch but converges to higher losses. The plot shows the average validation loss over 5 trials.}
    \label{fig:real_combined_loss}
\end{figure*}

\subsection{Real World Experiments}\label{sec:real_exp}
\subsubsection{Setup} We apply standard SGD, AdaBatch~\citep{devarakonda_adabatch_2018}, and \ourmethod to train standard ResNet-20~\citep{he2016deep} networks on \textsc{Cifar}-10, \textsc{Cifar}-100~\citep{krizhevsky2009learning}, and \textsc{Tiny-ImageNet}~\citep{le_tiny_nodate}.

\paragraph{Hyperparameters} We adapt the training code from open-sourced codebases\footnote{\url{https://github.com/weiaicunzai/pytorch-cifar100}\label{ft:100}}\footnote{\url{https://github.com/kuangliu/pytorch-cifar}\label{ft:10}}\footnote{\url{https://github.com/tjmoon0104/Tiny-ImageNet-Classifier}\label{ft:tiny_img}} that achieve state-of-the-art results the respective datasets, employing their hyperparameters. Please refer to the appendix for details.\ We choose 2048 as the maximum batch size in Adabatch and \ourmethod.

For each dataset, we run two versions of AdaBatch and \ourmethod: one with learning rate rescaling proportional to the batch size (e.g., doubling the batch size doubles the learning rate) and one without rescaling. This inclusion is motivated by the use of the rescaling version in AdaBatch. From the training loss curves, we observ that the version without learning rate rescaling shows more stable progression. Therefore, we report these results in the main text and include the learning rate rescaling results in \cref{app:rescaleTrue}.

\paragraph{Datasets}
\textsc{Cifar}-10 and \textsc{Cifar}-100 are two image classification datasets used in \citet{devarakonda_adabatch_2018}. \textsc{Cifar}-10 has 10 classes each with 6000 images. \textsc{Cifar}-10 has 100 classes each with 600 images. Both datasets are split into 50000 train and 10000 test images. Additionally, we include \textsc{Tiny-ImageNet}, which has 200 classes, each with 500 training images, and 50 validation images, to assess \ourmethod on a classification task with a larger number of classes. For the hyperparameter search on \ourmethod, we further hold out $20\%$ of the train images as validation sets.

\subsubsection{Results and Discussions}

\begin{table*}[ht!]
\centering
\begin{adjustbox}{max width=\textwidth}
\begin{tabular}{llccccc}
\toprule
& & \multicolumn{4}{c}{Validation Accuracy at Epochs (\% of Total Training)} & \multicolumn{1}{c}{Time to $\pm$1\% Final Acc.} \\
\cmidrule(lr){3-6} \cmidrule(lr){7-7}
 Dataset &Algorithm   & 25\% & 50\% & 75\%  & 100\% (Final) & Avg. Time (Seconds)\\
\midrule
\multirow{4}{*}{\textsc{Cifar}-10} & SGD (128) & 85.13 ± 0.36 & 88.88 ± 0.56 & 92.64 ± 0.20 & \textbf{95.50 ± 0.02} & 3188.80\\
 & SGD (2048) & 86.98 ± 0.27 & 91.11 ± 0.33 & 92.45 ± 0.24 & 92.47 ± 0.27 & 1556.89\\
 & AdaBatch (128 - 2048) & 89.55 ± 0.26 &\textbf{ 94.68 ± 0.02} & \textbf{94.86 ± 0.03} & 94.88 ± 0.02 & 1361.11  \\
 & \ourmethod (128 - 2048) & \textbf{92.16 ± 0.17} & 92.78 ± 0.11 & 93.81 ± 0.07 & 93.82 ± 0.08 & \textbf{638.15} \\
\midrule
\multirow{4}{*}{\textsc{Cifar}-100} & SGD (128) & 54.82 ± 1.22 & 66.26 ± 0.38 &\textbf{ 74.43 ± 0.17} &\textbf{74.92 ± 0.23} & 2808.46\\
 & SGD (2048) & 57.92 ± 0.80 & 67.01 ± 0.70 & 67.50 ± 0.66 & 67.59 ± 0.70 & \textbf{1097.55}\\
 & AdaBatch (128 - 2048) & 62.68 ± 0.80 & \textbf{72.75 ± 0.23} & 72.97 ± 0.22 & 72.99 ± 0.25 & 1262.71\\
 & \ourmethod (128 - 2048) &\textbf{ 68.70 ± 0.20} & 71.59 ± 0.25 & 71.72 ± 0.19 & 71.75 ± 0.16 & 1191.76 \\
\midrule
\multirow{4}{*}{\textsc{Tiny-ImageNet}} & SGD (256) & 45.34 ± 0.27 & \textbf{54.52 ± 0.23} & \textbf{55.40 ± 0.17} & \textbf{55.49 ± 0.18} & 9702.65\\
 & SGD (2048) & 22.52 ± 0.32 & 29.87 ± 0.08 & 30.97 ± 0.10 & 31.08 ± 0.06 & 11068.93\\
 & AdaBatch (256 - 2048) & 46.57 ± 0.25 & 53.25 ± 0.05 & 53.66 ± 0.12 & 53.72 ± 0.07 & 7146.14\\
 & \ourmethod (256 - 2048) & \textbf{49.66 ± 0.05} & 51.36 ± 0.04 & 51.66 ± 0.04 & 51.66 ± 0.06 & \textbf{6477.70}\\
\bottomrule
\end{tabular}
\end{adjustbox}
\caption{Validation accuracy and standard error for \textsc{CiFar}-10, \textsc{CiFar}-100, and \textsc{Tiny-ImageNet} at different stages of training, alongside the average epochs and time required to reach within ±1\% of the final validation accuracy. \ourmethod achieves the highest accuracy and lowest variance early in training and demonstrates superior efficiency. In addition, \ourmethod reaches within ±1\% of the final validation accuracy much faster than small-batch SGD and AdaBatch.}
\label{tab:epoch_acc_time}
\end{table*}

\paragraph{Faster Convergence with Trade-offs}  \Cref{fig:real_combined_acc} and \cref{fig:real_combined_loss}  demonstrate that \ourmethod converges more quickly than AdaBatch.\ The key advantage of \ourmethod is its rapid convergence, achieving high accuracy after approximately 30 epochs. However, this accelerated convergence comes with a trade-off: once \ourmethod stabilizes, the loss does not continue to decrease with additional epochs. Similarly, while \ourmethod reaches a high accuracy level quickly, further training does not yield significant improvements.

In \cref{tab:epoch_acc_time} we compare the validation accuracy of each algorithm at 25\%, 50\%, 75\%, and 100\% of the total training epochs. \ourmethod achieves the highest accuracy 25\% mark and maintains competitive accuracy at later stages, differing by less than $5\%$ from the highest values. We further analyze the number of epochs and the time (in seconds) each algorithm takes to achieve a validation accuracy within $\pm 1\%$ of its final validation accuracy. In \cref{tab:epoch_acc_time}, we see that \ourmethod takes less time than small-batch SGD and AdaBatch on all three datasets. Notably, for \textsc{CiFar}-10, \ourmethod is 2$\times$ faster than AdaBatch and 5$\times$ faster than small-batch SGD. This rapid convergence makes \ourmethod particularly well-suited for scenarios where training speed is more critical than ultimate model performance, such as the exploratory phase of neural network training. In such cases, practitioners can use \ourmethod to quickly identify a well-performing model, then switch to other training algorithms or model variants to further enhance performance.

%% file: contents/limitation.tex
\section{Limitations and Future Directions}
In this work, we introduce an adaptive batch size method based on gradient diversity. Although our method demonstrates empirical success, we do not yet fully understand its theoretical impact on convergence. Furthermore, \ourmethod approximates gradient diversity by accumulating gradient information over the mini-batches within an epoch. This approach may introduce estimation errors, as the gradients are computed with respect to parameters that are continuously updated throughout the epoch. Another limitation of \ourmethod is memory consumption. To compute the approximated gradient diversity, we need to compute the individual gradient within a minibatch using the Backpack-for-PyTorch library. In \cref{tab:memory_usage} we compare the peak memory usage of each method averaged over epochs when training on \textsc{CiFar}-10. A promising future direction is to theoretically justify this approximation or develop more accurate methods for estimating gradient diversity without adding computational overhead.

\begin{table}[ht!]
\centering
\begin{adjustbox}{max width=\columnwidth}
\begin{tabular}{lc}
\toprule
Algorithm & Peak Memory Usage (MB) \\
\midrule
SGD (128)& 717.18 \\
SGD (2048) & 9565.23 \\
AdaBatch (128 - 2048) & 6751.37 \\
\ourmethod (128 - 2048) & 13164.19 \\

\bottomrule
\end{tabular}
\end{adjustbox}
\caption{Average peak memory usage when training on \textsc{CiFar}-10.}
\label{tab:memory_usage}
\end{table}

Another direction is to integrate \ourmethod with methods that increase gradient diversity. For instance, \citet{yin_gradient_2018} shows that gradient quantization~\citep{alistarh_qsgd_2017} and stochastic gradient Langevin Dynamic~\citep{welling_bayesian_2011} promote gradient diversity of the entire dataset. We hypothesize that under mild conditions such as Lipschitz continuity, minibatch diversification~\citep{zhang_active_2018,zhang_determinantal_2017,zhdanov_diverse_2019} increases gradient diversity, therefore allowing larger batch sizes.

In this work, we focus on adaptive batch size optimization and do not directly compare \ourmethod with popular optimizers such as Adam~\citep{kingma_adam_2017} and AdamW~\citep{loshchilov2019decoupledweightdecayregularization}.  However, \ourmethod could complement these optimizers by adaptively determining batch sizes to further accelerate convergence.

%% file: contents/conclusion.tex
\section{Conclusion}
\label{sec:con}
We introduce \ourmethod, an adaptive batch size SGD algorithm rooted in the gradient diversity theory~\citep{yin_gradient_2018}. Our synthetic experiments confirm that when approximated gradient diversity aligns with true gradient diversity, the resulting optimization curves closely match, validating the effectiveness of our approximation. Moreover, both synthetic and real world experiments show that \ourmethod enables large-batch training where standard fixed batch SGD struggles. Notably, experiments on real datasets reveal that \ourmethod achieves faster convergence while maintaining accuracy comparable to small-batch SGD in substantially less time.

The time efficiency of \ourmethod highlights its potential to accelerate training processes, particularly during the early stages, where quickly identifying a well-performing model is essential. Future research could integrate \ourmethod as a preliminary step in training pipelines to further enhance efficiency and scalability.

%% file: contents/appendix.tex
\onecolumn
\section{Effect of Gradient Diversity and Batch Size}
\begin{lemma}[\citet{yin_gradient_2018}]\label{app:yin_lemma1}
    Let $\boldsymbol{\theta}^{t+1}$ be a fixed model, and let $\boldsymbol{\theta}^{t}$ denote the model after a mini-batch iteration with batch-size $m = \delta \cdot n \Delta_{\mathcal{S}}\left(\boldsymbol{\theta}^{t}\right)+1$. Then we have:
    \begin{equation*}
        \begin{aligned}
        & \mathbb{E}\left[\left\|\boldsymbol{\theta}^{t+1}-\boldsymbol{\theta}^*\right\|_2^2 \mid \boldsymbol{\theta}^t\right] \leq \left\|\boldsymbol{\theta}^t-\boldsymbol{\theta}^*\right\|_2^2  - m \cdot\left(2 \eta\left\langle\nabla_{\boldsymbol{\theta}} \mathcal{L}_{\mathcal{S}}\left(\boldsymbol{\theta}^t\right), \boldsymbol{\theta}^t-\boldsymbol{\theta}^*\right\rangle-(1+\delta)) \eta^2 M^2\left(\boldsymbol{\theta}^t\right)\right)
        \end{aligned}
    \end{equation*}
    where $M^2\left(\boldsymbol{\theta}^t \right) := \frac{1}{n}\sum_{i=1}^n \|\nabla_{\boldsymbol{\theta}} \ell\left(\boldsymbol{\theta}^t ; \boldsymbol{z}_i\right)\|^2_2$.
\end{lemma}

\begin{proof}
    \begin{equation}
    \begin{aligned}
        \mathbb{E}\left[\|\theta^{t+1} - \theta^*\|_2^2\middle| \theta^{t}\right] &= \mathbb{E}\left[\left\|\theta^{t} - \theta^* - \eta \sum_{j \in \mathcal{B}^t} \nabla_{\theta} \ell(\theta^t;z_j)\right\|_2^2 \middle| \theta^{t}\right]\\
         & = \|\theta^{t} - \theta^*\|_2^2 - 2\eta \mathbb{E}\left[\left\langle\theta^{t}-\theta^*, \sum_{j \in \mathcal{B}^t} \nabla_{\theta} \ell(\theta^t;z_j)\right\rangle \middle| \theta^{t}\right]
        + \eta^2\mathbb{E}\left[\left\|\sum_{j \in \mathcal{B}^t} \nabla_{\theta} \ell(\theta^t;z_j)\right\|_2^2 \middle| \theta^{t}\right]\\
        & = \|\theta^{t} - \theta^*\|_2^2 - 2\eta m \langle\theta^{t} -\theta^*, \nabla_{\theta} \mathcal{L}_{\mathcal{S}}(\theta^{t})\rangle + \eta^2 \mathbb{E}\left[ \sum_{j \in \mathcal{B}^t} \|\nabla_{\theta} \ell(\theta^t;z_j)\|^2_2 + \sum_{j \neq k}\langle \nabla_{\theta} \ell(\theta^t;z_j), \nabla_{\theta} \ell(\theta^t;z_j) \rangle \middle| \theta^{t}\right]\\
        & = \|\theta^{t} - \theta^*\|_2^2 - 2\eta m \langle\theta^{t} -\theta^*, \nabla_{\theta} \mathcal{L}_{\mathcal{S}}(\theta^{t})\rangle + \eta^2 \frac{m}{n} \sum_{i = 1}^n \|\nabla_{\theta} \ell(\theta^t;z_i)\|^2_2 \\
        & \quad + \eta^2 \mathbb{E}\left[\sum_{j \neq k}\langle \nabla_{\theta} \ell(\theta^t;z_j), \nabla_{\theta} \ell(\theta^t;z_k) \rangle \middle| \theta^{t}\right]\\
        & = \|\theta^{t} - \theta^*\|_2^2 - 2\eta m \langle\theta^{t} -\theta^*, \nabla_{\theta} \mathcal{L}_{\mathcal{S}}(\theta^{t})\rangle + \eta^2 \frac{m}{n} \sum_{i = 1}^n \|\nabla_{\theta} \ell(\theta^t;z_i)\|^2_2 \\
        & \quad + \eta^2 \mathbb{E}\left[ \sum_{j \neq k}\langle \nabla_{\theta} \ell(\theta^t;z_j), \nabla_{\theta} \ell(\theta^t;z_k) \rangle \middle| \theta^{t}\right]\\
        & = \|\theta^{t} - \theta^*\|_2^2 - 2\eta m \langle\theta^{t} -\theta^*, \nabla_{\theta} \mathcal{L}_{\mathcal{S}}(\theta^{t})\rangle + \eta^2 \frac{m}{n} \sum_{i = 1}^n \|\nabla_{\theta} \ell(\theta^t;z_i)\|^2_2 \\
        & \quad + \eta^2 m (m - 1) \left\|\frac{1}{n} \sum_{i = 1}^n \nabla_{\theta} \ell(\theta^t;z_i)\right\|^2_2\\
        & \text{Define } M^2(\theta) := \frac{1}{n} \sum_{i=1}^n \|\nabla_{\theta} \ell(\theta^t;z_i)\|^2 \text{ and } G(\theta) := \|\nabla_{\theta} \mathcal{L}_{\mathcal{S}}(\theta)\|^2_2 = \left\|\frac{1}{n} \sum_{i = 1}^n \nabla_{\theta} \ell(\theta^t;z_i)\right\|^2_2\\
        &= \|\theta^{t} - \theta^*\|_2^2 - 2\eta m \langle\theta^{t} -\theta^*, \nabla_{\theta} \mathcal{L}_{\mathcal{S}}(\theta^{t})\rangle \nonumber + \eta^2 (m M^2(\theta^{t}) +  m(m - 1)G(\theta^{t})) \\
        & \text{Observe that } n \Delta_{\mathcal{B}^t}\left(\theta^t\right) = \frac{M^2(\theta)}{G(\theta)} \\
        &= \|\theta^{t} - \theta^*\|_2^2 - 2\eta m\langle\theta^{t} -\theta^*, \nabla_{\theta} \mathcal{L}_{\mathcal{S}}(\theta^{t})\rangle \nonumber  + \eta^2  m \left( 1 + \frac{ m - 1}{n \Delta_{\mathcal{B}^t}\left(\theta^t\right)} \right) M^2(\theta^{t}) \\
        &= \|\theta^{t} - \theta^*\|_2^2 - m \left( 2\eta \langle\theta^{t} -\theta^*, \nabla_{\theta} \mathcal{L}_{\mathcal{S}}(\theta^{t})\rangle - \eta^2 (1 + \delta)M^2(\theta^{t})\right).
    \end{aligned}
\end{equation} 
\end{proof}\label{app:lemma1}

\section{Convergence Analysis}\label{app:conv}
Existing convergence analysis of mini-batch SGD guarantees the convergence of \ourmethod. We include the convergence analysis adapted from \citet{bottou_optimization_2018}.

\begin{assumption}[\textbf{Lipschitz-continuous gradients}]\label{assumption:lipschitz_grad} The function $\mathcal{L}$ : $\mathbb{R}^d \rightarrow \mathbb{R}$ is continuously differentiable and the gradient function of $\mathcal{L}$,  is Lipschitz continuous with Lipschitz constant $L>0$, i.e.,
$$
\|\nabla \mathcal{L}(\theta)-\nabla \mathcal{L}(\bar{\theta})\|_2 \leq L\|\theta-\bar{\theta}\|_2 \quad \text { for all }\theta, \bar{\theta} \in \mathbb{R}^d
$$
\end{assumption}
~\\For the ease of notation, we denote $g(\theta, \mathbf{s}) := \frac{1}{|\mathbf{s}|} \sum_{i = 1}^{|\mathbf{s}|} \nabla_{\theta} \ell(\theta; \mathbf{s}_i)$. Note that in mini-batch size is not fixed, so the following analying applies to \ourmethod.
\begin{assumption}[\textbf{First and second-moment limits}]\label{assumption:moment_bound} The objective function $\mathcal{L}$ and the mini-batch SGD algorithm satisfy the following:
\begin{enumerate}
\item The sequence of parameters $\left\{\theta^t\right\}$ is contained in an open set over which $\mathcal{L}$ is bounded below by a scalar $\mathcal{L}_{\mathrm{inf}}$.
\item There exist scalars $\mu_G \geq \mu>0$ such that, for all $t \in \{1,2, \dots\}$,
$$
\begin{aligned}
\nabla \mathcal{L}\left(\theta^t\right)^{\top} \mathbb{E}_{\mathbf{s}_t}\left[g\left(\theta^t, \mathbf{s}_t\right)\right] & \geq \mu\left\|\nabla \mathcal{L}\left(\theta^t\right)\right\|_2^2  \text {  and  } 
\left\|\mathbb{E}_{\mathbf{s}_t}\left[g\left(\theta^t, \mathbf{s}_t\right)\right]\right\|_2 & \leq \mu_G\left\|\nabla \mathcal{L}\left(\theta^t\right)\right\|_2
\end{aligned}
$$ \label{assumption2.2}
    \item  There exist scalars $M \geq 0$ and $M_V \geq 0$ such that, for all $k \in \mathbb{N}$,
$$
\mathbb{V}_{\mathbf{s}_t}\left[g\left(\theta^t, \mathbf{s}_t\right)\right] \leq M+M_V\left\|\nabla \mathcal{L}\left(\theta^t\right)\right\|_2^2,
$$
where $\mathbb{V}_{\mathbf{s}_t}\left[g\left(\theta^t, \mathbf{s}_t\right)\right]$ is the variance of $g\left(\theta^t, \mathbf{s}_t\right)$.
\end{enumerate}
\end{assumption}

~\\Having the necessary assumptions, we next introduce 2 lemmas on which the convergence guarantees are based.

\begin{lemma}\label{lem:lip}
    Under \cref{assumption:lipschitz_grad}, the iterates of (mini-batch) SGD satisfy the following for all $t \in \mathbb{N}$,
    \begin{equation}
\mathbb{E}_{\mathbf{s}_t}\left[\mathcal{L}\left(\theta^{T+1}\right)\right]-\mathcal{L}\left(\theta^t\right) \leq -\alpha_t \nabla \mathcal{L}\left(\theta^t\right)^{\top} \mathbb{E}_{\mathbf{s}_t}\left[g\left(\theta^t, \mathbf{s}_t\right)\right]+\frac{1}{2} \alpha_t^2 L \mathbb{E}_{\mathbf{s}_t}\left[\left\|g\left(\theta^t, \mathbf{s}_t\right)\right\|_2^2\right].
\end{equation}
\end{lemma}

\begin{proof}
We first show a consequence of \cref{assumption:lipschitz_grad}, 

\begin{equation}
\begin{aligned}
F(w) & =F(\bar{w})+\int_0^1 \frac{\partial F(\bar{w}+t(w-\bar{w}))}{\partial t} d t \quad {\text{(Fundamental Theorem of Calculus)}}\\
& =F(\bar{w})+\int_0^1 \nabla \mathcal{L}(\bar{w}+t(w-\bar{w}))^T(w-\bar{w}) d t \\
& =F(\bar{w})+\nabla \mathcal{L}(\bar{w})^T(w-\bar{w})+\int_0^1[\nabla \mathcal{L}(\bar{w}+t(w-\bar{w}))-\nabla \mathcal{L}(\bar{w})]^T(w-\bar{w}) d t \\
& \leq F(\bar{w})+\nabla \mathcal{L}(\bar{w})^T(w-\bar{w})+\int_0^1 L\|t(w-\bar{w})\|_2\|w-\bar{w}\|_2 d t \\
& = F(\bar{w})+\nabla \mathcal{L}(\bar{w})^T(w-\bar{w})+\frac{1}{2} L\|w-\bar{w}\|_2^2.
\end{aligned}\label{eq:1}
\end{equation}
By \cref{eq:1}, the iterates generated by (mini-batch) SGD satisfy
\begin{equation}
\begin{aligned}
\mathcal{L}\left(\theta^{T+1}\right)-\mathcal{L}\left(\theta^t\right) & \leq \nabla \mathcal{L}\left(\theta^t\right)^{\top}\left(w_{T+1}-\theta^t\right)+\frac{1}{2} L\left\|w_{T+1}-\theta^t\right\|_2^2 \\
& \leq-\alpha_t \nabla \mathcal{L}\left(\theta^t\right)^{\top} g\left(\theta^t, \mathbf{s}_t\right)+\frac{1}{2} \alpha_t^2 L\left\|g\left(\theta^t, \mathbf{s}_t\right)\right\|_2^2
\end{aligned}
\end{equation}Taking expectations on all sides with respect to the distribution of $s_k$,
we have

\begin{equation}
\begin{aligned}
\mathbb{E}{\left[\mathcal{L}\left(\theta^{T+1}\right)-\mathcal{L}\left(\theta^t\right)\right]} & \leq \mathbb{E}{\left[\nabla \mathcal{L}\left(\theta^t\right)^{\top}\left(w_{T+1}-\theta^t\right)+\frac{1}{2} L\left\|w_{T+1}-\theta^t\right\|_2^2\right]} \\
& \leq \mathbb{E}{\left[-\alpha_t \nabla \mathcal{L}\left(\theta^t\right)^{\top} g\left(\theta^t, \mathbf{s}_t\right)+\frac{1}{2} \alpha_t^2 L\left\|g\left(\theta^t, \mathbf{s}_t\right)\right\|_2^2\right]}\\
& = -\alpha_t \nabla \mathcal{L}\left(\theta^t\right)^{\top} \mathbb{E}_{\mathbf{s}_t}\left[g\left(\theta^t, \mathbf{s}_t\right)\right]+\frac{1}{2} \alpha_t^2 L \mathbb{E}_{\mathbf{s}_t}\left[\left\|g\left(\theta^t, \mathbf{s}_t\right)\right\|_2^2\right],
\end{aligned}
\end{equation}
as desired.
\end{proof}

~\\We have the following inequalities under additional \cref{assumption:moment_bound}.
\begin{lemma}\label{lem:lip_bd_moment}
Under \cref{assumption:lipschitz_grad} and \cref{assumption:moment_bound}, the iterates of (mini-batch) SGD satisfy the following for all $t \in \mathbb{N}$,

\begin{equation}
\begin{aligned}
\mathbb{E}_{\mathbf{s}_t}\left[\mathcal{L}\left(\theta^{T+1}\right)\right]-\mathcal{L}\left(\theta^t\right) & \leq-\mu \alpha_t\left\|\nabla \mathcal{L}\left(\theta^t\right)\right\|_2^2+\frac{1}{2} \alpha_t^2 L \mathbb{E}_{\mathbf{s}_t}\left[\left\|g\left(\theta^t, \mathbf{s}_t\right)\right\|_2^2\right] \\
& \leq-\left(\mu-\frac{1}{2} \alpha_t L M_G\right) \alpha_t\left\|\nabla \mathcal{L}\left(\theta^t\right)\right\|_2^2+\frac{1}{2} \alpha_t^2 L M,
\end{aligned}
\end{equation}
where $M_G := M_V + \mu^2_G$.
\end{lemma}

\begin{proof}
We start by introducing a consequence of \cref{assumption:moment_bound}, that is, the second moment of $g(\theta^t, \mathbf{s}_t)$ satisfies 
\begin{equation}\label{eq4.9}
\mathbb{E}_{\mathbf{s}_t}\left[\left\|g\left(\theta^t, \mathbf{s}_t\right)\right\|_2^2\right] \leq M+M_G\left\|\nabla \mathcal{L}\left(\theta^t\right)\right\|_2^2
\end{equation}
By \cref{lem:lip} and \cref{assumption:moment_bound}(2), 
    \begin{equation}
\begin{aligned}
\mathbb{E}_{\mathbf{s}_t}\left[\mathcal{L}\left(\theta^{T+1}\right)\right]-\mathcal{L}\left(\theta^t\right) & \leq-\alpha_t \nabla \mathcal{L}\left(\theta^t\right)^{\top} \mathbb{E}_{\mathbf{s}_t}\left[g\left(\theta^t, \mathbf{s}_t\right)\right]+\frac{1}{2} \alpha_t^2 L \mathbb{E}_{\mathbf{s}_t}\left[\left\|g\left(\theta^t, \mathbf{s}_t\right)\right\|_2^2\right] \\
& \leq-\mu \alpha_t\left\|\nabla \mathcal{L}\left(\theta^t\right)\right\|_2^2+\frac{1}{2} \alpha_t^2 L \mathbb{E}_{\mathbf{s}_t}\left[\left\|g\left(\theta^t, \mathbf{s}_t\right)\right\|_2^2\right] \\
& \leq -\mu \alpha_t\left\|\nabla \mathcal{L}\left(\theta^t\right)\right\|_2^2 + \frac{1}{2} \alpha_t^2 L \left( M+M_G\left\|\nabla \mathcal{L}\left(\theta^t\right)\right\|_2^2\right)\\
& = -\left(\mu-\frac{1}{2} \alpha_t L M_G\right) \alpha_t\left\|\nabla \mathcal{L}\left(\theta^t\right)\right\|_2^2+\frac{1}{2} \alpha_t^2 L M, 
\end{aligned}
\end{equation}
where the third inequality comes from \cref{eq4.9}.
\end{proof}

~\\Now we are ready to show the convergence analysis of (mini-batch) SGD for general loss functions. Here we focus on general loss functions as they are ubiquitous in modern machine learning. Convex and strongly convex functions have also been widely studied, we refer interested readers to \citet{bottou_optimization_2018, gower_sgd_2019,robbins_stochastic_1951, bottou_large-scale_2010, garrigos_handbook_2024}.

\begin{theorem}[\textbf{Fixed Stepsize}]\label{thm:nonconvex_fixstep}
    Under \cref{assumption:lipschitz_grad} and \cref{assumption:moment_bound}, and suppose the (mini-batch) SGD runs with a fixed stepsize $\alpha = \alpha_t \; \forall t$, such that
    \begin{equation}\label{eq:fixstepsize_bound}
        0 \leq \alpha \leq \frac{\mu}{L M_G}, 
    \end{equation}
we have the (mini-batch) SGD iterates satisfy the following inequalities for all $T \in \mathbb{N}$:

\begin{subequations}\label{eq:fixstep}
\begin{align}
& \mathbb{E}\left[\sum_{t=1}^T\left\|\nabla \mathcal{L}\left(\theta^t\right)\right\|_2^2\right] \leq \frac{T \alpha L M}{\mu}+\frac{2\left(\mathcal{L}\left(\theta^1\right)-\mathcal{L}_{\text {inf }}\right)}{\mu \alpha} \label{eq:fixstep_a}\\
\implies & \mathbb{E}\left[\frac{1}{T} \sum_{t=1}^T\left\|\nabla \mathcal{L}\left(\theta^t\right)\right\|_2^2\right] \leq \frac{\alpha L M}{\mu}+\frac{2\left(\mathcal{L}\left(\theta^1\right)-F_{\mathrm{inf}}\right)}{T \mu \alpha} \xrightarrow{K \rightarrow \infty} \frac{\alpha L M}{\mu} \label{eq:fixstep_b}.
\end{align}
\end{subequations}
\end{theorem}

\begin{proof}
From \cref{lem:lip_bd_moment}, we have 
\begin{equation}
\begin{aligned}
\mathbb{E}_{\mathbf{s}_t}\left[\mathcal{L}\left(\theta^{T+1}\right)\right]-\mathcal{L}\left(\theta^t\right) & \leq-\left(\mu-\frac{1}{2} \alpha L M_G\right) \alpha\left\|\nabla \mathcal{L}\left(\theta^t\right)\right\|_2^2+\frac{1}{2} \alpha^2 L M,
\end{aligned}
\end{equation}
Taking expectations with respect to the randomness of the parameter $\theta^t$, 

\begin{equation}
\begin{aligned}
\mathbb{E}_{\theta^t} \left[\mathbb{E}_{\mathbf{s}_t}\left[\mathcal{L}\left(\theta^{T+1}\right)\right]-\mathcal{L}\left(\theta^t\right)\mid \theta^t \right]
& = 
\mathbb{E}\left[\mathcal{L}\left(\theta^{T+1}\right)\right]-\mathbb{E}\left[\mathcal{L}\left(\theta^t\right)\right] \\
& \leq-\left(\mu-\frac{1}{2} \alpha L M_G\right) \alpha \mathbb{E}\left[\left\| \nabla \mathcal{L}\left(\theta^t\right)\right\|_2^2 \right]+\frac{1}{2} \alpha^2 L M,
\end{aligned}
\end{equation}
Summing both size over $t = 1, ..., T$, we have a telescoping sum, together with \cref{eq:fixstepsize_bound}, that simplifies to:

\begin{equation}
\begin{aligned}
\mathbb{E}\left[\mathcal{L}\left(\theta^{T+1}\right)\right]-\mathbb{E}\left[\mathcal{L}\left(\theta^1\right)\right] \leq-\frac{1}{2} \mu \alpha \sum_{t=1}^T\mathbb{E}\left[\left\|\nabla \mathcal{L}\left(\theta^t\right)\right\|_2^2\right]+\frac{1}{2} T \alpha^2 L M
\end{aligned}
\end{equation}

~\\Now by \cref{assumption:moment_bound}(3),

\begin{equation}
\begin{aligned}
\mathcal{L}_{\text {inf }}-\mathcal{L}\left(\theta^1\right) &\leq \mathbb{E}\left[\mathcal{L}\left(\theta^{T+1}\right)\right]-\mathbb{E}\left[\mathcal{L}\left(\theta^1\right)\right] \leq-\frac{1}{2} \mu \alpha \sum_{t=1}^T\mathbb{E}\left[\left\|\nabla \mathcal{L}\left(\theta^t\right)\right\|_2^2\right]+\frac{1}{2} T \alpha^2 L M
 \\
& \iff 
\mathbb{E}\left[\sum_{t=1}^T\left\|\nabla \mathcal{L}\left(\theta^t\right)\right\|_2^2\right] \leq \frac{T \alpha L M}{\mu}+\frac{2\left(\mathbb{E}\left[\mathcal{L}\left(\theta^1\right)\right]-\mathcal{L}_{\text {inf }}\right)}{\mu \alpha}
\end{aligned}
\end{equation}
~\\ Divide both sides by $T$,
\begin{equation}
\mathbb{E}\left[\frac{1}{T} \sum_{t=1}^T\left\|\nabla \mathcal{L}\left(\theta^t\right)\right\|_2^2\right] \leq \frac{\alpha L M}{\mu}+\frac{2\left(\mathbb{E}\left[\mathcal{L}\left(\theta^1\right)\right]-F_{\mathrm{inf}}\right)}{T \mu \alpha}
\end{equation}
~\\where the second term vanishes as $T \rightarrow \infty$.
\end{proof}
~\\Next we consider decreasing stepsizes $\alpha_t$ that satisfy the Robbin's and Monro condition~\citep{robbins_stochastic_1951}, i.e,

\begin{equation}\label{eq:robbin_monro}
\sum_{k=1}^{\infty} \alpha_t=\infty \text { and } \sum_{k=1}^{\infty} \alpha_t^2<\infty
\end{equation}

\begin{theorem}[\textbf{Diminishing Stepsizes}] Under \cref{assumption:lipschitz_grad} and \cref{assumption:moment_bound}, suppose that the (mini-batch) SGD  uses a stepsize sequence satisfying \cref{eq:fixstepsize_bound} and \cref{eq:robbin_monro}, then, with $A_T:=\sum_{t=1}^T\alpha_t$,
$$
\lim _{K \rightarrow \infty} \mathbb{E}\left[\sum_{t=1}^T\alpha_t\left\|\nabla \mathcal{L}\left(\theta^t\right)\right\|_2^2\right]<\infty
$$
$$
\implies \mathbb{E}\left[\frac{1}{A_T} \sum_{t=1}^T\alpha_t\left\|\nabla \mathcal{L}\left(\theta^t\right)\right\|_2^2\right] \xrightarrow{K \rightarrow \infty} 0 \text {. }
$$
\end{theorem}\label{thm:dimiinsh_step}
\begin{proof}
    The condition in \cref{eq:robbin_monro} guarantees that $\alpha_t \rightarrow 0$. Taking the expectations on both sides of \cref{lem:lip_bd_moment},
\begin{equation}\label{eq:diminish_step}
\begin{aligned}
\mathbb{E}\left[\mathcal{L}\left(\theta^{T+1}\right)\right]-\mathbb{E}\left[\mathcal{L}\left(\theta^t\right)\right] 
& \leq-\left(\mu-\frac{1}{2} \alpha_t L M_G\right) \alpha_t \mathbb{E}\left[\left\| \nabla \mathcal{L}\left(\theta^t\right)\right\|_2^2 \right]+\frac{1}{2} \alpha_t^2 L M\\
& \leq-\frac{1}{2} \mu \alpha_t \mathbb{E}\left[\left\|\nabla \mathcal{L}\left(\theta^t\right)\right\|_2^2\right]+\frac{1}{2} \alpha_t^2 L M
\end{aligned}
\end{equation}
~\\Summing both sides for $t \in \{1, \dots, T\}$ yields the telescoping sum:
\begin{equation}
\begin{aligned}
\mathbb{E}\left[\mathcal{L}\left(\theta^{T+1}\right)\right]-\mathbb{E}\left[\mathcal{L}\left(\theta^1\right)\right] \leq-\frac{1}{2} \mu \sum_{t=1}^T\alpha_t  \mathbb{E}\left[\left\|\nabla \mathcal{L}\left(\theta^t\right)\right\|_2^2\right]+\frac{1}{2} L M \sum_{t=1}^T\alpha_t^2
\end{aligned}
\end{equation}

~\\Again, by \cref{assumption:moment_bound}(3),

\begin{equation}
\begin{aligned}
\mathcal{L}_{\text {inf }}-\mathcal{L}\left(\theta^1\right) &\leq \mathbb{E}\left[\mathcal{L}\left(\theta^{T+1}\right)\right]-\mathcal{L}\left(\theta^1\right) \leq-\frac{1}{2} \mu \sum_{t=1}^T\alpha_t  \mathbb{E}\left[\left\|\nabla \mathcal{L}\left(\theta^t\right)\right\|_2^2\right]+\frac{1}{2} L M \sum_{t=1}^T\alpha_t^2
 \\
& \iff 
\sum_{t=1}^T\alpha_t \mathbb{E}\left[\left\|\nabla \mathcal{L}\left(\theta^t\right)\right\|_2^2\right] \leq \frac{L M}{\mu} \sum_{t=1}^T\alpha^2_t +\frac{2\left(\mathbb{E}\left[\mathcal{L}\left(\theta^1\right)\right]-\mathcal{L}_{\text {inf }}\right)}{\mu}
\end{aligned}
\end{equation}
~\\The second term on the right-hand side converges to a finite value by \cref{eq:robbin_monro}, proving the first part of \cref{eq:diminish_step}.

~\\For the second part of \cref{eq:diminish_step}, notice that $A_T \rightarrow \infty$ as $T \rightarrow \infty$, so  
\begin{equation}
\mathbb{E}\left[\frac{1}{A_T} \sum_{t=1}^T\alpha_t\left\|\nabla \mathcal{L}\left(\theta^t\right)\right\|_2^2\right] \xrightarrow{T \rightarrow \infty} 0 .
\end{equation}
\end{proof}

\section{Details on Synthetic Experiments}
In \cref{tab:syn_hparams}, we list the hyperparameters used for the synthetic experiments. Notably, we maintain the ratio of the learning rate to batch size as $\frac{\eta^{sgd}}{m^{sgd}}$, are the optimal values for small-batch SGD, determined through grid search.

\begin{table}[h!]
\centering
\begin{tabular}{llcccc}
\toprule
&\textbf{Hyperparameter} & \textbf{SGD (128)} & \textbf{SGD (4096)} & \textbf{\ourmethod}& \textbf{\oracle} \\
\midrule
\multirow{6}{*}{\textbf{Convex}} & Learning rate & 16 & 512 & 16 & 16\\
 & Initial batch size & 128 & 4096 & 128& 128 \\
 & $\frac{\text{Batch size}}{\text{Learning rate}}$ & 8 & 8 & 8 & 8\\
 & Gradient diversity scale $\delta$ & -- & -- & 1 & 1 \\
 & Maximum batch size & 256 & 4096 & 4096 & 4096\\
 & Learning rate decay & 0.75 & 0.75 & 0.75  & 0.75\\
\midrule
 &\textbf{Hyperparameter} & \textbf{SGD (512)} & \textbf{SGD (5028)} & \textbf{\ourmethod} & \textbf{\oracle}\\
\midrule
\multirow{6}{*}{\textbf{Nonconvex}} & Learning rate & 1 & 9.83 & 1 & 1 \\
 & Initial batch size & 512 & 5028 & 512 & 512 \\
 & $\frac{\text{Batch size}}{\text{Learning rate}}$ & 512 & 512 & 512 & 512\\
 & Gradient diversity scale $\delta$ & -- & -- & 0.1 & 0.1\\
 & Maximum batch size & 256 & 5028 & 5028 & 8192 \\
 & Learning rate decay & 0.75 & 0.75 & 0.75 & 0.75 \\
\bottomrule
\end{tabular}
\caption{Hyperparameters of SGD and \ourmethod in synthetic experiments.}
\label{tab:syn_hparams}
\end{table}

\section{Details on Real-World Experiments}\label{appedix:exp_detail}
\paragraph{Hyperparameters}
The hyperparameters for real-world dataset experiments are summarized in \cref{tab:resnet_hparams}. We referred interested readers to \cref{ft:100}, \cref{ft:10}, and \cref{ft:tiny_img}.

\begin{table*}[ht!]
\centering
\begin{adjustbox}{max width=\textwidth}
\begin{tabular}{llcccc}
\toprule
\textbf{Dataset} & \textbf{Hyperparameter} &\textbf{SGD (small-batch)} & \textbf{SGD (large batch)} & \textbf{AdaBatch} & \textbf{\ourmethod} \\
\midrule
\multirow{7}{*}{\textsc{CiFar}-10} & Initial learning rate & 0.1 & 0.1 & 0.1 & 0.1 \\
 & Initial batch size & 128 & 2048 & 128 & 128 \\
 & Maximum batch size & 128 & 2048 & 2048 & 2048 \\
 & Resize factor & -- & -- & $\times$2 & -- \\
 & Resize frequency & -- & -- & 20 & 20 \\
 & Gradient Diversity Scale $\delta$ & -- & -- & -- & 0.1 \\
\midrule
\multirow{7}{*}{\textsc{CiFar}-100} & Initial learning rate & 0.1 & 0.1 & 0.1 & 0.1 \\
 & Initial batch size & 128 & 2048 & 128 & 128 \\
 & Maximum batch size & 128 & 2048 & 2048 & 2048 \\
 & Resize factor & -- & -- & $\times$2 & -- \\
 & Resize frequency & -- & -- & 20 & 20 \\
 & Gradient Diversity Scale $\delta$ & -- & -- & -- & 0.01 \\
\midrule
\multirow{7}{*}{\textsc{Tiny-ImageNet}} & Initial learning rate & 0.01 & 0.01 & 0.01 & 0.01 \\
 & Initial batch size & 256 & 2048 & 256 & 256 \\
 & Maximum batch size & 256 & 2048 & 2048 & 2048 \\
 & Resize factor & -- & -- & $\times$2 & -- \\
 & Resize frequency & -- & -- & 20 & 20 \\
 & Gradient Diversity Scale $\delta$ & -- & -- & -- & 0.01 \\

\bottomrule
\end{tabular}
\end{adjustbox}
\caption{Hyperparameters for ResNet models across different training methods on \textsc{CiFar}-10, \textsc{CiFar}-100, and \textsc{Tiny-ImageNet}.}
\label{tab:resnet_hparams}
\end{table*}

\section{Additional Experiments--Rescaling Learning Rate}\label{app:rescaleTrue}
In this version of the real-world experiment, we rescale the learning proportional to batch size. For example, when running SGD (2048) on \textsc{CiFar}-10, we start with an initial learning rate (2048/128) * 0.1. Similarly, when AdaBatch or \ourmethod rescale the batch size by $k$ times, the learning rate is also multiplied by $k$. Below we report the results analogous to those in \cref{sec:real_exp}. From \cref{fig:real_combined_acc_rescale} and \cref{fig:real_combined_loss_rescale}, we observe that rescaling the learning rate leads to unstable training trajectory in both \textsc{CiFar}-10 and \textsc{CiFar}-100.

\begin{figure*}[t!]
    \centering
    \begin{subfigure}[b]{0.33\textwidth}
        \centering
        \includegraphics[width=\textwidth]{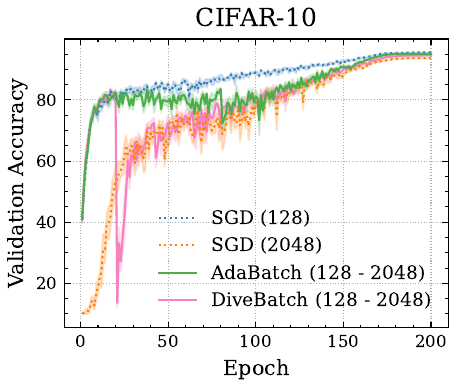}
    \end{subfigure}\hspace{0pt}
    \begin{subfigure}[b]{0.33\textwidth}
        \centering
        \includegraphics[width=\textwidth]{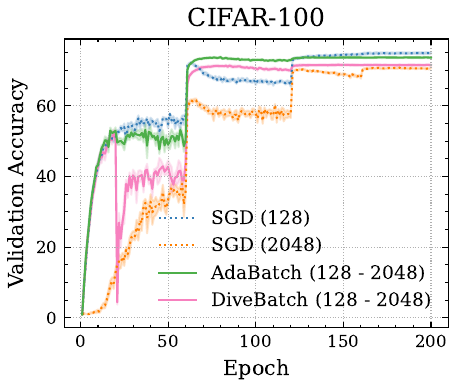}
    \end{subfigure}\hspace{0pt}
    \begin{subfigure}[b]{0.33\textwidth}
    \centering
    \includegraphics[width=\textwidth]{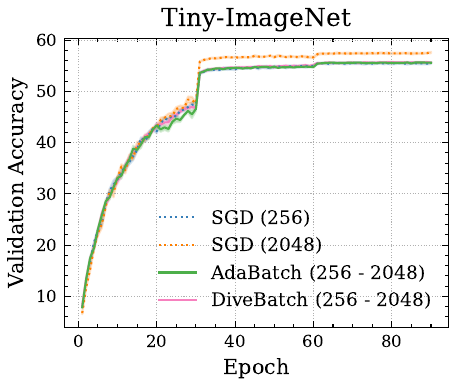}
    \end{subfigure}\hspace{15pt}
    \caption{Validation accuracy on \textsc{CiFar}-10, \textsc{CiFar}-100 and \textsc{Tiny-ImageNet}. The plots show the average validation accuracy over 5 trials.}
    \label{fig:real_combined_acc_rescale}
    \centering
    \begin{subfigure}[b]{0.33\textwidth}
        \centering
        \includegraphics[width=\textwidth]{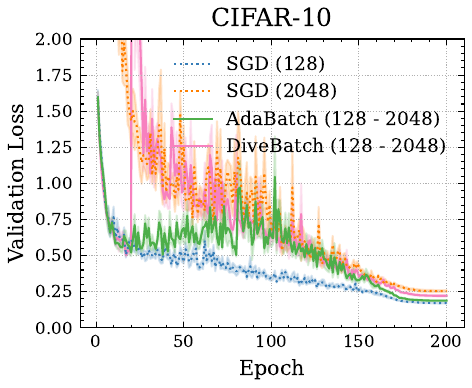}
    \end{subfigure}\hspace{6pt}
    \begin{subfigure}[b]{0.31\textwidth}
        \centering
        \includegraphics[width=\textwidth]{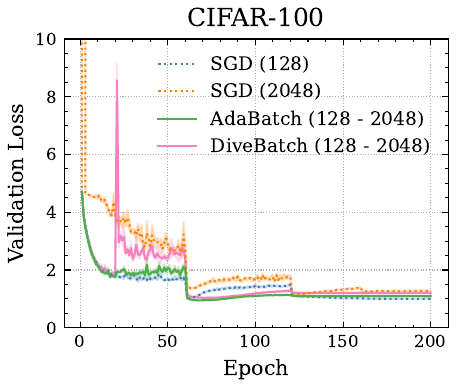}
    \end{subfigure}\hspace{0pt}
    \begin{subfigure}[b]{0.32\textwidth}
    \centering
    \includegraphics[width=\textwidth]{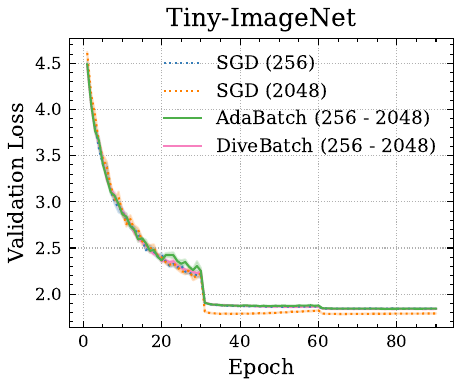}

    \end{subfigure}
    \caption{Validation loss on \textsc{CiFar}-10, \textsc{CiFar}-100 and \textsc{Tiny-ImageNet}. The loss of \ourmethod stabilizes faster than AdaBatch but converges to higher losses. The plot shows the average validation loss over 5 trials.}
    \label{fig:real_combined_loss_rescale}
\end{figure*}

\begin{table*}[ht!]
\centering
\begin{adjustbox}{max width=\textwidth}
\begin{tabular}{llccccc}
\toprule
 &  & \multicolumn{4}{c}{Validation Accuracy at Epochs (\% of Total Training)} & \multicolumn{1}{c}{Time to $\pm$1\% Final Acc.} \\
\cmidrule(lr){3-6} \cmidrule(lr){7-7}
 Dataset & Algorithm & 25\% & 50\% & 75\%  & 100\% (Final) & Avg. Time (Seconds) \\
\midrule
\multirow{4}{*}{\textsc{CiFar}-10} 
 & SGD (128) & 85.13 ± 0.36 & 88.88 ± 0.56 & 92.64 ± 0.20 & 95.50 ± 0.02 & 3188.80\\
 & SGD (2048) & 65.32 ± 2.50 & 79.09 ± 0.67 & 88.10 ± 0.36 & 93.64 ± 0.12 & 2601.14\\
 & AdaBatch (128 - 2048) & 79.47 ± 2.10 & 81.13 ± 2.06 & 90.66 ± 0.50 & 95.05 ± 0.11 & 2694.04\\
 & DiveBatch (128 - 2048) & 72.50 ± 2.56 & 79.44 ± 1.89 & 90.06 ± 0.34 & 94.55 ± 0.10 & \textbf{1073.11}\\
\midrule
\multirow{4}{*}{\textsc{CiFar}-100} 
 & SGD (128) & 54.82 ± 1.22 & 66.26 ± 0.38 & 74.43 ± 0.17 & 74.92 ± 0.23 & 2808.46\\
 & SGD (2048) & 33.36 ± 1.20 & 57.83 ± 0.65 & 68.88 ± 0.36 & 70.75 ± 0.15 & 2151.14\\
 & AdaBatch (128 - 2048) & 49.00 ± 1.19 & 73.00 ± 0.11 & 73.69 ± 0.11 & 73.64 ± 0.13 &1298.14\\
 & DiveBatch (128 - 2048) & 42.78 ± 1.76 & 70.63 ± 0.19 & 71.58 ± 0.23 & 71.53 ± 0.18 &\textbf{1346.40} \\

\midrule
\multirow{4}{*}{\textsc{Tiny-ImageNet}} 
 & SGD (256) & 45.34 ± 0.27 & 54.52 ± 0.23 & 55.40 ± 0.17 & 55.49 ± 0.18 & 9702.65\\
 & SGD (2048) & 44.85 ± 0.18 & 56.91 ± 0.14 & 57.41 ± 0.10 & 57.57 ± 0.15 & 8857.43\\
 & AdaBatch (256 - 2048) & 42.99 ± 0.52 & 54.75 ± 0.07 & 55.55 ± 0.08 & 55.58 ± 0.11 &8764.02 \\
 & DiveBatch (256 - 2048) & 43.93 ± 0.33 & 54.71 ± 0.11 & 55.54 ± 0.09 & 55.60 ± 0.10 &\textbf{8399.57}\\

\bottomrule
\end{tabular}
\end{adjustbox}
\caption{Validation accuracy and standard error for \textsc{CiFar}-10, \textsc{CiFar}-100, and \textsc{Tiny-ImageNet} at different stages of training, alongside the average epochs and time required to reach within ±1\% of the final validation accuracy.}
\label{tab:epoch_acc_time_rescaleTrue}
\end{table*}